\documentclass[11pt, reqno]{amsart}
\usepackage[margin=3.6cm]{geometry}
\usepackage{hyperref}
\usepackage{array,booktabs,tabularx}
\usepackage{graphicx}
\usepackage{amssymb}
\usepackage{enumerate}
\usepackage{color}
\usepackage{esint}
\usepackage{mathtools}
\usepackage{dsfont}
\usepackage{ esint }
\usepackage{enumitem}
\usepackage{comment}

\newtheorem{theorem}{Theorem}

\newtheorem{proposition}[theorem]{Proposition}
\newtheorem{lemma}[theorem]{Lemma}

\newtheoremstyle{named}{}{}{\itshape}{}{\bfseries}{.}{.5em}{\thmnote{#3 }#1} \theoremstyle{named} 

\theoremstyle{definition}
\newtheorem{definition}{Definition}
\newtheorem{remark}{Remark}

\newcommand{\R}{{\mathbb R}}

\newcommand{\ep}{\varepsilon}

\newcommand{\gives}{\ensuremath{\rightarrow}}

\newcommand{\setst}[2]{\ensuremath{ \left\{ #1\,\right|\left.\,#2 \right\}}}

\newcommand{\abs}[1]{\ensuremath{\left| #1 \right|}}
\newcommand{\lr}[1]{\ensuremath{\left(#1 \right)}}
\newcommand{\norm}[1]{\left\lVert#1\right\rVert}

\newcommand{\w}{\omega}

\newcommand{\set}[1]{\ensuremath{\{#1\}}}

\def\XXint#1#2#3{{\setbox0=\hbox{$#1{#2#3}{\int}$} \vcenter{\hbox{$#2#3$}}\kern-.5\wd0}}

\DeclareMathOperator{\Relu}{ReLU}
\DeclareMathOperator{\diam}{diam}

\title[Relu Nets of Minimal Width]{Approximating Continuous Functions by ReLU Nets of Minimal Width}

\author[Boris Hanin]{Boris Hanin}
\author[Mark Sellke]{Mark Sellke}

\address[B. Hanin]{Department of Mathematics, Texas A\&M, College Station,
  United States\medskip}
\email{bhanin@math.tamu.edu}
\address[M. Sellke]{Trinity College, Cambridge, CB2 1TQ, UK} 
\email{mas279@cam.ac.uk}
\begin{document}
\begin{abstract} This article concerns the expressive power of depth in deep feed-forward neural nets with $\Relu$ activations. Specifically, we answer the following question: for a fixed $d_{in}\geq 1,$ what is the minimal width $w$ so that neural nets with $\Relu$ activations, input dimension $d_{in}$, hidden layer widths at most $w,$ and arbitrary depth can approximate any continuous, real-valued function of $d_{in}$ variables arbitrarily well? It turns out that this minimal width is exactly equal to $d_{in}+1.$ That is, if all the hidden layer widths are bounded by $d_{in}$, then even in the infinite depth limit, $\Relu$ nets can only express a very limited class of functions, and, on the other hand, any continuous function on the $d_{in}$-dimensional unit cube can be approximated to arbitrary precision by $\Relu$ nets in which all hidden layers have width exactly $d_{in}+1.$ Our construction in fact shows that any continuous function $f:[0,1]^{d_{in}}\to\mathbb R^{d_{out}}$ can be approximated by a net of width $d_{in}+d_{out}$. We obtain quantitative depth estimates for such an approximation in terms of the modulus of continuity of $f$. 
\end{abstract}
\maketitle

\section{Introduction}
Over the past several years, artificial neural networks, especially deep
networks, have become the state of the art in a wide variety of
machine learning tasks. These tasks include important benchmark
problems in machine vision (\cite{krizhevsky2012imagenet}) and machine translation
(\cite{sutskever2014sequence, wu2016google}) as well as superhuman performance at games such as Go
\cite{silver2016mastering}. Despite these varied and striking
successes, a theory of why neural nets provide such good 
approximations to interesting functions and can be effectively trained
is only beginning to take shape. 

While non-linear activations help neural nets express a wide variety
of functions, repeated non-linearities can also
``garble'' the signal, leading to a loss of mutual information between
the input and the activations at various hidden layers. Such an
information theoretic point of view on neural nets has recently been
systematically taken up in the work of Tishby with Shwartz-Ziv,
Moshkovitz, and Zaslavsky \cite{shwartz2017opening,
  moshkovitz2017mixing, tishby2015deep}. In the present article, we answer a
basic information theoretic question about neural nets. Namely, for
each $d\geq 1,$ what
is the minimal width $w_{\text{min}}(d)$ so that neural nets whose
hidden layers have width at least $w_{\text{min}}(d)$ and arbitrary
depth can approximate arbitrarily well any scalar continuous
function of $d$ variables? We treat only neural nets
with a popular and particularly simple activation function called
rectified linear units, defined 
\[\Relu(t):=\max\lr{0,t}.\]

It have been known since the 1980's (e.g. the work of Cybenko \cite{cybenko1989approximation} and
Hornik-Stinchcombe-White \cite{hornik1989multilayer}) that feed-forward neural
nets with a single hidden layer can approximate essentially any
function if the hidden layer is allowed to be arbitrarily
wide. Such results hold for a wide variety of activations, including
 $\Relu.$ However, part of the recent renaissance in neural nets, is the empirical observation 
 that deep neural nets tend to achieve greater expressivity per
 parameter than their shallow cousins. There are now a number of
 rigorous results about this so-called expressive power of depth
\cite{arora2016understanding,mhaskar2016learning,
  lin2017does,mhaskar2016deep,poole2016exponential,raghu2016expressive,
  rolnick2017power, telgarsky2015representation,
  telgarsky2016benefits, telgarsky2017neural, yarotsky2016error}. We
refer the reader to \S 3 in \cite{hanin2017universal} for a discussion
of the relationships between some of these articles. 

The main result of this article shows a sharp transition in the
representational power of deep feed-forward neural nets with
$\Relu$ activations as a function of the widths of their hidden layers. To state it, we need some
notation. We say that $\mathcal N$ is a feed-forward neural net with
$\Relu$ activations, input dimension $d_{in}$, output dimension
$d_{out}$, and widths $d_{in}=d_1,d_2,\ldots,d_k,d_{k+1}=d_{out} $ (a $\Relu$ net for short)
if it computes a function $f_{\mathcal N}$ of the form
\begin{equation}\label{E:relunet-def}
A_k\circ \Relu\circ A_{k-1}\circ \cdots \circ \Relu\circ
A_1,
\end{equation}
where $A_i:\R^{d_i}\gives \R^{d_{i+1}}$ are affine transformations and
for any $m\geq 1$ 
\[\Relu\lr{x_1,\ldots, x_m}=\lr{\max\lr{0,x_1},\ldots,
  \max\lr{0,x_m}}.\]
The integers $d_2,\ldots, d_k$ are said to be the widths of the hidden
layers of $\mathcal N,$ and the integer $k$ is the depth of $\mathcal
N.$ Notice that for fixed $d_1,\ldots d_{k+1},$ the family of neural
nets \eqref{E:relunet-def} is a finite dimensional family of
non-linear functions parameterized by the affine transformations
$A_i.$ Our main result concerns the numbers $w_{\text{min}}(d_{in},d_{out}),$ defined to 
be the minimal value of $w$ such that for every continuous function
$f:[0,1]^{d_{in}}\gives\R^{d_{out}}$ and every $\ep>0$ there is a $\Relu$ net
$\mathcal N$ with input dimension $d_{in},$ hidden layer widths at most $w$, and
output dimension $d_{out}$ that $\ep-$approximates $f:$
\[\sup_{x\in [0,1]^{d_{_{in}}}}\norm{f(x)-f_{\mathcal N}(x)}\leq \ep.\]
The main result of this article is the following estimate for $w_{\text{min}}(d_{in}, d_{out}).$
\begin{theorem}\label{T:main}
For every $d_{in},d_{out}\geq 1,$ 
\[d_{in}+1\leq w_{\text{min}}(d_{in},d_{out})\leq d_{in}+d_{out}.\]
\end{theorem}
\noindent Proving the upper bound $w_{\text{min}}(d_{in},d_{out})\leq d_{in}+d_{out}$ in Theorem
\ref{T:main} requires a novel construction by which any continuous
function with $d_{in}$ input variables and $d_{out}$ output variables can be approximated to arbitrary precision by a $\Relu$ net with width $d_{in}+d_{out}$ and depth depending on its modulus of continuity $\w_f$. Recall that $\w_f(\delta)\leq\varepsilon$ when $|x-y|\leq \delta$ implies that $|f(x)-f(y)|\leq \varepsilon$ uniformly over all inputs $x,y.$ Since $\w_f$ need not be continuous or bijective, define \[\w^{-1}_f(\varepsilon)=\sup\{\delta:\w_f(\delta)\leq\varepsilon\}.\] We will show that if $K\subseteq \R^{d_{in}}$ is any compact set and
$f:K\gives \R^{d_{out}}$ is continuous, then there exists a $\Relu$ net
$\mathcal N$ with input dimension $d_{in}$, all hidden layers of width
$d_{in}+d_{out}$, output dimenion $d_{out},$ and depth
$O(\diam(K)/\w^{-1}_f(\ep))^{d_{in}+1}$ that $\ep$-approximates $f$ on $K:$
\[\sup_{x\in K} \norm{f(x)-f_{\mathcal N}(x)}\leq \ep.\]
We refer the reader to Proposition \ref{P:density-maxmin} for the
precise statement. The construction is carried out in \S \ref{S:UB}. In contrast, obtaining the lower bound 
\[w_{\text{min}}(d_{in},d_{out})\geq w_{\text{min}}(d_{in},1)\geq d_{in}+1,\]
requires constructing, for every $d_{in}\geq 1,$ a continuous function
$f:[0,1]^{d_{in}}\gives \R$ and a constant $\eta>0$ so that any width $d$
$\Relu$ net $\mathcal N$ must satisfy  
\begin{equation}\label{E:LB}
\sup_{x\in[0,1]^{d_{in}}}\abs{f(x)-f_{\mathcal N}(x)}> \eta.
\end{equation}
Our construction in \S \ref{S:LB} only requires that the function have some compact level set (connected component of a fiber $f^{-1}(a)$) and be non-constant inside that level set. 

Before proceeding to the proof of Theorem \ref{T:main}, we make
two remarks. First, the neural nets we consider here are not allowed to
have skip (e.g. residual) connections, popularized in the ResNets
introduced by He-Zhang-Ren-Sun in \cite{he2016deep} and in the Highway
Nets introduced by Srivastava-Greff-Schidhuber in \cite{srivastava2015highway}. A skip connection allows the input to a given hidden layer to be an affine function of the all the outputs of all the previous hidden layers, instead of just the one preceeding it. If one allows skip connections, then a $\Relu$ net whose hidden layers have width $1$ can already approximate any continuous function if the net is allowed to be arbitrarily deep. The reason is that any feed-forward neural net with one hidden layer of width $k$ can be converted into a neural net with $k$ hidden layers, each of width $1,$ that computes the same function. The construction is simply to ``turn the hidden layer on its' side.'' That is, each neuron in the single hidden layer in the original shallow net becomes its own hidden layer. The input to the net is connected to the single neural in every new hidden layer, which is in turn connected to the output. In this construction, each hidden layer is connected only to the input and output. In the language of Veit-Wilber-Belongie \cite{veit2016residual}, the resulting ResNet implements an ensemble of
paths of length $1$. Second, it is tempting to generalize Theorem
\ref{T:main} to arbitrary piecewise linear activations. However, it
seems that such a generalization is not straightforward, even for activations of the form $\sigma(t)=\max\lr{\ell_1(t),\ell_2(t)}$, where $\ell_1,\ell_2$ are two affine functions with different slopes.

\subsection*{Acknowledgements}
The first author would like to thank Zhangyang Wang for several
stimulating discussions about extending the results in this article to
allowing residual connections and to more general activations. We are also grateful to Dmitry Yarotsky for pointing out several inaccuracies and a mistake (now corrected) in the proof of Lemma \ref{L:LB-key} in a previous version. 

\section{Proof of the Upper Bound in Theorem \ref{T:main}}\label{S:UB}
Fix $\ep>0,$ $d_{in}, d_{out}\geq 1$, a compact set $K\subseteq \R^{d_{in}},$ and a
continuous function $f:K\gives \R^{d_{out}}.$ In this section, we
prove that there exists a ReLU net $\mathcal N$ with input dimension $d_{in},$ hidden layer widths
$d_{in}+d_{out},$ and output dimension $d_{out}$ such that 
\begin{equation}\label{E:goal1}
\norm{f-f_{\mathcal N}}_{C^0(K)}=\sup_{x\in K}\norm{f(x)-f_{\mathcal N}(x)}\leq \ep.
\end{equation}
We will use the following definition.
\begin{definition}
  A function $g:\R^{d_{in}}\gives \R^{d_{out}}$ is a max-min string of length $L\geq 1$ on
  $d_{in}$ input variables and $d_{out}$ output variables if there exist  affine functions $\ell_1,\ldots,
  \ell_{L}:\R^{d_{in}}\gives \R^{d_{out}}$ such that 
\[g=\sigma_{L-1}(\ell_{L},\sigma_{L-2}(\ell_{L-1},\ldots,
\sigma_2(\ell_3, \sigma_1(\ell_1,\ell_2))\cdots),\]
where each $\sigma_i$ is either a coordinate-wise max or a min. 
\end{definition}
\noindent The statement \eqref{E:goal1} follows immediately from the following two
propositions.
\begin{proposition}\label{P:relu-maxmin}
  For every max-min string $g$ on $d_{in}$ input variables and $d_{out}$ ouput variables with length $L$ and
  every compact $K\subseteq \R^{d_{in}}$, there
  exists a $\Relu$ net with input dimension $d,$ hidden
  layer width $d_{in}+d_{out},$ output dimension $d_{dout}$, and depth $L$ that computes
  $x\mapsto g(x)$ for every $x\in K.$
\end{proposition}

\begin{proposition}\label{P:density-maxmin}
  For every compact $K\subseteq \R^{d_{in}},$ any continuous
  $f:K\gives \R^{d_{out}}$ and each $\ep>0$ there exists a
  max-min string $g$ on $d_{in}$ input variables and $d_{out}$ output variables with length 
\[\left(\frac{O(\diam(K))}{\w^{-1}_f(\ep)}\right)^{d_{in}+1}\]
for which
\[\norm{f-g}_{C^0(K)}\leq \ep.\]
\end{proposition}

\noindent Proposition \ref{P:relu-maxmin} is essentially Lemma 4 in
\cite{hanin2017universal}. We include a short proof for the reader's convenience in \S
\ref{S:relu-pf}. Proposition \ref{P:density-maxmin} appears to be new,
however, and is the main technical result in the present
article. It is proved in \S \ref{S:density-pf}. It is related in spirit
to results in the literature (e.g. \cite[Prop. 2.2.2.]{scholtes2012introduction}) that express a continuous
piecewise affine $h:K\gives \R$ on a convex domain as 
\[\max_{1\leq i \leq N}\min_{1\leq j\leq M(i)}
\set{\ell_{1,i},\ldots,\ell_{M(i),i}},\qquad \ell_{j,i}:K\gives
\R\quad \text{affine}.\]
Nonetheless, Proposition \ref{P:density-maxmin} is of a rather
different nature since we are allowed to take only max and min of two
affine functions at a time. 

\subsection{Proof of Proposition
  \ref{P:relu-maxmin}}\label{S:relu-pf} 
We may assume without loss of generality that $K$ is contained in the
positive orthant:
\[K\subseteq \R_+^{d_{in}}=\setst{\lr{x_1,\ldots, x_{d_{in}}}\in \R^{d_{in}}}{x_i\geq
  0,\qquad 1\leq i \leq d_{in}}\]
since we can always shift the input to a neural net by a fixed
vector. Let us fix a max-min string
\[g=\sigma_{L-1}(\ell_{L},\sigma_{L-2}(\ell_{L-1},\ldots,
\sigma_2(\ell_3, \sigma_1(\ell_1,\ell_2))\cdots).\]
We can assume $g$ is non-negative since we can subtract a constant in the final linear transformation.
Note that for any constant $C,$ the function $g+C$ is also a max-min
string whose affine tranformations are $\ell_i+C.$ Since we may
subtract an arbitrary constant in the output of the last 
layer in a $\Relu$ net, we may additionally assume that each $\ell_i$
is non-negative on $K.$ With these reductions, we construct the neural
net that computes $g(x)$ for every $x\in K.$ For all $j=2,\ldots, L$ define  invertible affine tranformations $A_j:\R^{d_{in}+d_{out}}\gives \R^{d_{in}+d_{out}}$ 
\[A_j(x,y)=
\begin{cases}
   (x, y-\ell_j(x)), &\text{if   }\sigma_{j-1}=\text{max}\\
   (x, -y+\ell_j(x)), &\text{if   }\sigma_{j-1}=\text{min}.
\end{cases},
\]
where $x\in \R^{d_{in}}$ and $y\in \R^{d_{out}}.$ Their inverses are given by
\[A_j^{-1}(x,y) =
\begin{cases}
   (x, y+\ell_j(x)), &\text{if   }\sigma_{j-1}=\text{max}\\
  (x, -y+\ell_j(x)), &\text{if   }\sigma_{j-1}=\text{min}.
\end{cases}. 
\] 
Further, set
\[A_1(x)=(x,\ell_1(x)),\qquad x\in \R^{d_{in}}.\]
Write $H_1:=A_1$ and
\[H_j:=A_j\circ \Relu \circ A_j^{-1},\qquad j=2,\ldots, L.\]
The image of $K$ under $H_0$ is the graph of $\ell_1,$ and
the image of the graph of any function $g:K\gives \R^{d_{out}}$ under $H_j$ is
the graph of $\sigma_{j-1}\lr{\ell_j,g}.$ Hence, the image of $K$ under
the $\Relu$ net  
\[\Relu\circ H_L\circ\cdots \circ H_1\]
is the graph of $g.$ Note that the final $\Relu$ is trivial since $g$
is non-negative. Appending a final layer $(x_1,\ldots,x_{d_{in}}, y_1,\ldots, y_{d_{out}}
)\mapsto \lr{y_1,\ldots, y_{d_{out}}}$ yields the desired net.
\qed

\subsection{Proof of Proposition \ref{P:density-maxmin}}\label{S:density-pf}
Note that if $g$ is a max-min string on $d_{in}$ input variables and $d_{out}$ output variables, then so is
$g(x-x_0)$ for any $x_0\in \R^{d_{in}}.$ Using also that every
compact set is contained in a ball shows that we may assume without
loss of generality that $K$ is a ball $B_r$ of radius $r$ centered at
the origin. 

Fix a continuous function $f:B_r\gives \R^{d_{out}}.$ We first explain how to
uniformly approximate $f$ by max-min strings in the model case when we
seek to approximate $f$ on an arbitrary finite subset of $\mathbb R^{d_{in}}$.  

\begin{proposition}\label{P:discrete}
Let $S\subseteq\mathbb R^{d_{in}}$ be a finite set. Then any function $f:S\to
\R^{d_{out}}$ can be computed exactly by a max-min string. 
\end{proposition}

\begin{proof}
We prove the proposition by induction on $|S|$. If $S=\set{s},$
then the constant max-min string $f(s)$ suffices. Suppose now that $\abs{S}\geq 2.$ The idea is
to consider the convex hull $\widehat{S}$ of the points in $S$
and ``repeatedly cut off a corner.'' Let $s_0\in S$ be an extreme point of $\widehat S$, a vertex of $\widehat{S}$ that is not contained in any proper face. By the inductive hypothesis, there is a max-min
string $g$ on $d_{in}$ input variables and $d_{out}$ output variables that agrees with $f$ on $S\backslash\{s_0\}$. Moreover, for every $t>0,$ we can find an affine function
$\ell:\R^{d_{in}}\gives \R^{d_{out}}$ with $\ell(s_0)=0$ and $\ell(s)\geq t$ for $s\in S\backslash
\set{s_0}$ (the inequality holds for each of the $k$ components of $\ell$). Taking $t$ large, define the max-min string
\[\widehat{g}=\max(\min(g, f(s_0)+\ell),f(s_0)-\ell)  ,\]
where the max and min are componentwise. By construction, $\widehat{g}(s_0)=f(s_0).$ Further, because $t$ is large, $\widehat{g}(s)=f(s)$ for $s\in S\backslash \set{ s_0}$. Hence $\widehat{g}$ and $f$ agree on $S$, completing the proof.
\end{proof}

We carry out the same proof idea for continuous functions on
$\R^{d_{in}}.$ We focus for simplicity on the construction for $d_{in}=2$ and $d_{out}=1$. The extension to general $d_{out}$ is immediate and requires only that various inequalities below hold for every component of vectors in $\R^{d_{out}}.$ The extension to $d_{in}\geq 3$ requires a minor modification, which we present after the $d_{in}=2$ proof. Before getting
into the details, we emphasize the main difference 
between the discrete case treated in Proposition \ref{P:discrete}
above and the continuous case below. The issue is that now when we cut
off a corner from the convex hull of the set where we have
$\ep$-approximated the function $f$, we have to approximate $f$
correctly on the entire piece we cut off, not just at a single
vertex. To get an $\varepsilon$-approximation, we need our corner piece
to have diameter $O\left(\w^{-1}_f(\ep)\right)$ so that the variation of $f$
on the piece is $O(\varepsilon)$ (recall that $\w^{-1}_f(\ep)$ is
the inverse modulus of continuity). That is, we can only cut off
small-diameter pieces at a time. Thus, to build an approximation to
$f$ on ball of radius $R$ from an approximation to $f$ on a ball of
radius $r<R$, we have to slowly add small pieces to $B_r$ in
all directions until the resulting set grows to contain $B_R.$ Our
precise construction repeatedly uses the following observation. We
state the observation for $d_{in}=2$ and explain below its extension to
$d_{in}\geq 3.$

\begin{lemma} \label{L:extend}
Fix $\ep>0$ and a continuous function $f:\mathbb R^2\to\mathbb
R$. Suppose
$K\subseteq \R^2$ and $\triangle ABC$ is an triangle with
\[\diam(\triangle ABC)\leq\w^{-1}_f(\ep)\]
such that $K$ is contained in the infinite planar sector
$\angle BAC$. Then if there exists a max-min string $g$ with
\[\sup_{x\in K}\abs{f(x)-g(x)}\leq \ep,\] 
then there also exists a max-min string $\widehat{g}$ with 
\[\sup_{x\in K\cup ABC}\abs{f(x)-\widehat{g}(x)}\leq \ep.\] 
\end{lemma}

\begin{proof}
Let $g$ be a max-min string that $\varepsilon$-approximates $f$ on
$K$. Let $\ell$ be the affine function with $\ell(A)=0$,
$\ell(B)=\ell(C)=\varepsilon$. As in Proposition \ref{P:discrete}, define
\[\ell_-(x):=f(A)-\ell(x),\qquad \ell_+(x):=f(A)+\ell(x)\]
and consider the max-min string
\[\widehat{g}=\max(\ell_-, \min(\ell_+, g)).\]
Next, by the definition of $\w^{-1}_f,$ we have 
\[\abs{f(x)-\widehat{g}(x)-f(x)}\leq \ep,\qquad \qquad x\in ABC.\]
We now show that this estimate continues to hold for $x\in K$ as
well. We claim that on $K\cup\triangle ABC$ we have   
\begin{equation}
\ell_--\ep\leq f\leq \ell_++\ep \label{eq:linbound}
\end{equation}
These inequalities follow essentially from the fact that the absolute
values of the slopes of $\ell_\pm$ on the rays $Ap$ for any point $p$ on segment $BC$ are
bounded below by $\frac{\w^{-1}_f(\ep)}{\ep}$. This is true since for any such $p$ we have $\ell(p)=\varepsilon$ and $|A-p|\leq\omega_f^{-1}(\varepsilon)$. Now to prove Inequality~\ref{eq:linbound} we note that for any $x\in K$ we have 
\[|f(x)-f(A)|\leq \ep+\frac{\ep}{\w^{-1}_f(\ep)}|x-A|.\] Indeed, suppose that ray $Ax$ has length $n\omega^{-1}_f(\ep)+r$ for integer $n$ and real remainder $r<\omega^{-1}_f(\ep)$. Then by the triangle inequality we have \[|f(x)-f(A)|\leq n\ep+\ep \leq \ep+\frac{\ep}{\w^{-1}_f(\ep)}|x-A|\] as desired.

 These estimates imply that on $K$
\[f-\ep=\min(f, f-\ep)\leq \min(\ell_+,g)\leq
\max\lr{\ell_-,\min(\ell_+,g)}=\widehat{g}\]
and
\[\widehat{g}=\max\lr{\ell_-,\min(\ell_+,g)}\leq \max(\ell_-, g)\leq \max(f, f+\ep)=f+\ep.\]
Therefore, $\widehat{g}-\ep\leq f \leq \widehat{g}+\ep,$ as desired. 
\end{proof}

\begin{figure}[!ht]
\includegraphics[width=10cm]{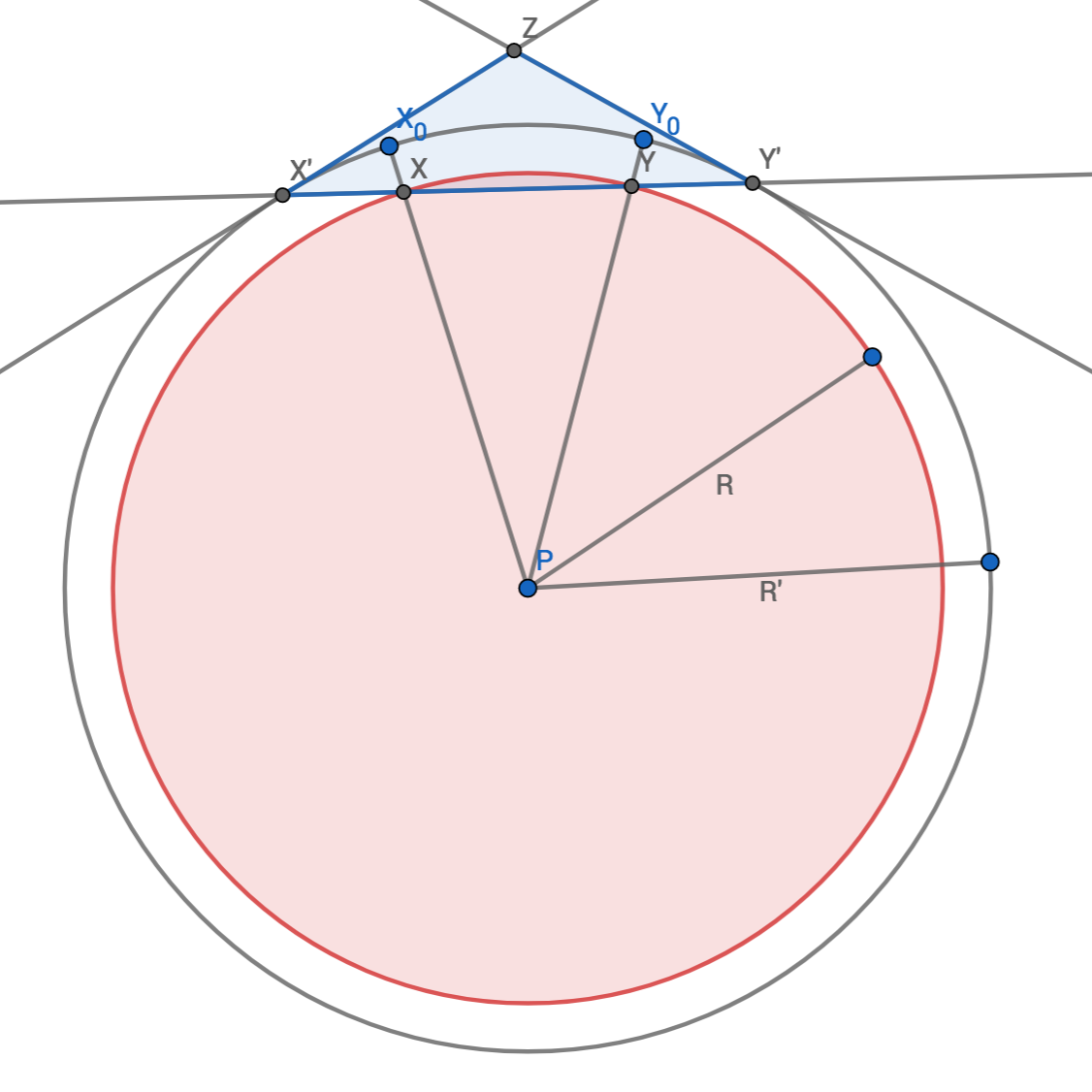}
\centering
\caption{To extend an $\varepsilon$-approximation of $f$ on the inner disk of radius $r$ to the outer disk of radius $r'=r+\frac{\omega_f^{-1}(\varepsilon)^2}{r}$, we proceed in steps. Each step, we draw triangle $X'ZY'$ as shown and apply Lemma~\ref{L:extend} to extend our approximation to a larger region. Because the outer circle $B_{r'}(P)$ is contained in sector $X'ZY'$, we do not lose any area contained in $B_{r'}(P)$ when applying Lemma~\ref{L:extend}.} 
\label{fig:extension}
\end{figure}

\begin{figure}[!ht]
\includegraphics[width=10cm]{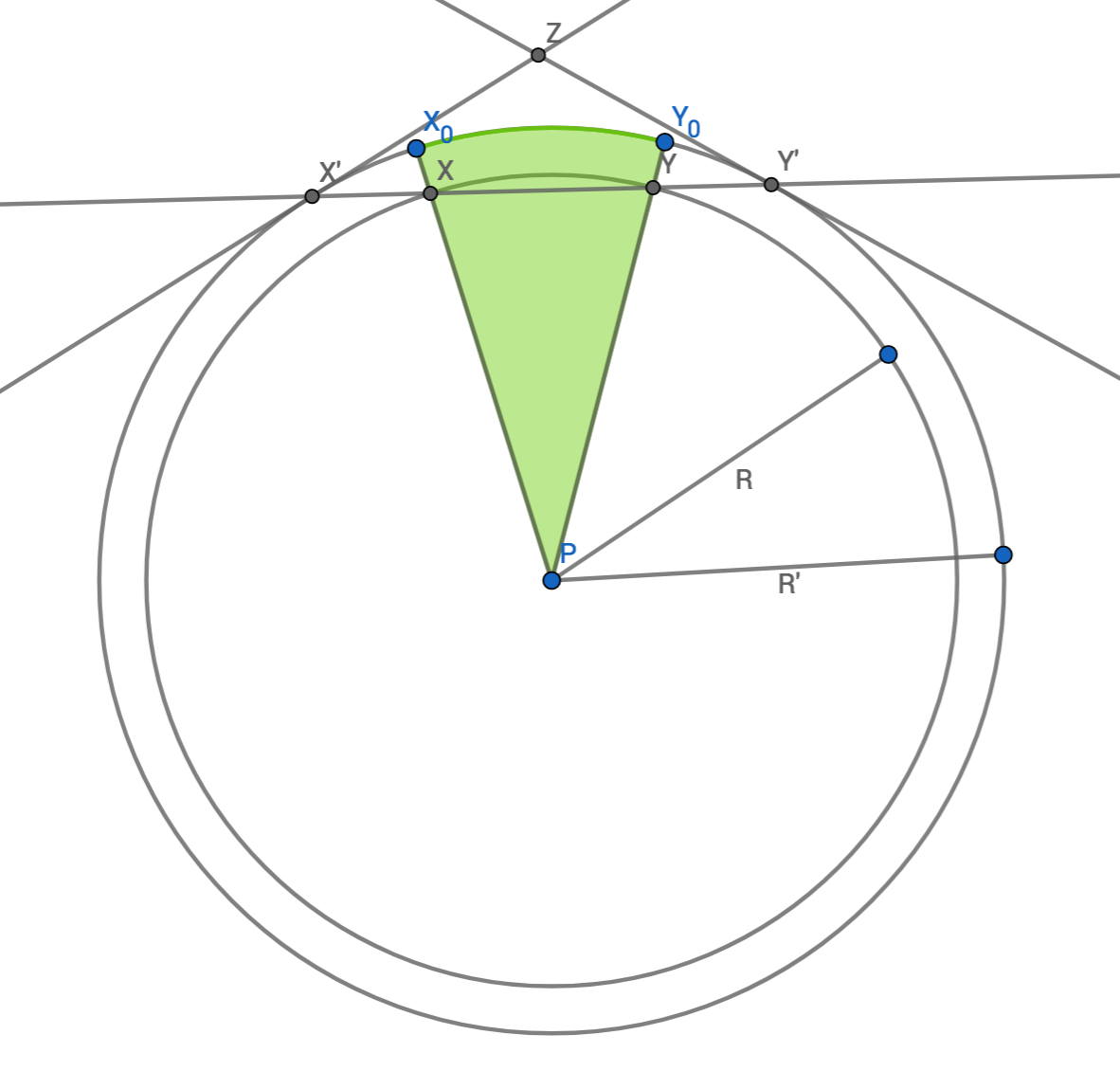}
\centering
\caption{In Figure~\ref{fig:extension}, after applying Lemma~\ref{L:extend}, the region on which we approximated $f$ has grown to include the shaded circular sector $X_0PY_0$. (This is just because it is contained in the union of the two shaded regions in Figure~\ref{fig:extension}.) Since $d(X,Y)\asymp \varepsilon$, this means that applying Lemma~\ref{L:extend} to $O\left(\frac{r}{\varepsilon}\right)$ rotated configurations of this form extends the region of $\varepsilon$-approximation from $B_r(P)$ to $B_{r'}(P)$.} 
\label{fig:extension2}
\end{figure}

We now turn to the details of the proof of
Proposition \ref{P:density-maxmin}. We will explain how to approximate
our fixed continuous function $f$ by 
a max-min string on a ball of radius $R>0$ centered at the
origin. We will use Lemma~\ref{L:extend} to
show that we can approximate $f$ on successively larger and larger
balls. Observe that if $r\leq \w^{-1}_f(\ep)$ then 
\[\norm{f-f(0)}_{C^0(B_r)}\leq \ep,\]
so that the constant max-min string $f(0)$ is an $\varepsilon$-approximation to $f$ on the small ball $B_{\w^{-1}_f(\ep)}(0)$. 
To prove that we can approximate $f$ on larger balls, suppose $g$ is a
max-min string on ${d_{in}}$ variables that approximates $f$ to within
$\varepsilon$ on the ball $B_r(0)$ with $r\geq w_f^{-1}(\ep).$ We use
Lemma~\ref{L:extend} to construct a new max-min string $\widehat{g}$
which uniformly $\varepsilon$-approximates $f$ on a ball of slightly
larger radius 
\[R_{r,\ep}:=r+\frac{\w^{-1}_f(\ep)^2}{10r}.\] 
Since for every $\ep>0,$ the function $R_{r,\ep}$ is strictly increasing in $r$ it cannot have a fixed point and the $k-$fold composition $R_{r,\ep}^{(k)}$ sends any $r>0$ to infinity with $k.$ Using this procedure repeatedly therefore allows us to increase $r$ without bound and will complete the proof. Our approach is illustrated in
Figures~\ref{fig:extension} and \ref{fig:extension2}.  

We begin with the construction when ${d_{in}}=2$ and will explain the simple
modification for ${d_{in}}\geq 3$ below. For each $r'>r$ and any two
sufficiently close points $X,Y$ on the 
boundary of $B_r$, let $X',Y'$ be the intersections of line $XY$ with
the boundary circle of $B_{r'}(P)$. Also, denote by $Z$ be the
intersection of the tangents to $B_{r'}$ through $X',Y'$ (see Figure 1). Then $B_r$
is contained in the planar sector $\angle X'ZY'$, and the diameter of
$\triangle X'ZY'$ can be made arbitrarily small by taking $r'$ close
to $r$ and $X$ close to $Y.$ In particular, for every $r\geq
\w^{-1}_f(\ep),$ we take 
\[r'=R_{r,\ep}=r+\frac{\w^{-1}_f(\ep)^2}{10 r},\qquad
\abs{XY}=\w^{-1}_f(\ep)\lr{1-\frac{\w^{-1}_f(\ep)^2}{100 r^2}}^{1/2}.\]
This choice for $|XY|$ is valid because $|XY|\leq \omega^{-1}_f(\ep)\leq r$, so we can indeed find points $X,Y$ at this distance with no problem.
Then
\[|X'Y'|=\sqrt{|XY|^2+4(r'-r)^2}=\w^{-1}_f(\ep).\]
We also know that $|X'Z|=|Y'Z|\leq |X'Y'|=\w^{-1}_f(\ep)$ because $|X'Y'|=\w^{-1}_f(\ep)\leq r\leq r'$, implying obtuseness of $\triangle X'ZY'$ at $Z$.
Thus, 
\[\diam(\triangle X'ZY')=|X'Y'|=\w^{-1}_f(\ep).\] 
Lemma \ref{L:extend} therefore shows that there exists a max-min string $g'$
that uniformly $\ep$ approximates $f$ on $K'=\triangle X'Y'Z\cup B_r.$ Notice that $K'$
contains the circular sector of $B_{R_{r,\ep}}$ cut out by the rays $OX$ and
$OY.$ Finally, consider a $\delta-$net $\set{p_i}$ on
the circumference of $B_r$ with
\[\delta =\frac{\w^{-1}_f(\ep)}{2}\lr{1-\frac{\w^{-1}_f(\ep)^2}{100  r^2}}^{1/2.}\] 
The size of this net is $O(r/\w^{-1}_f(\ep)).$ Applying Lemma
\ref{L:extend} $O(r/\w^{-1}_f(\ep))$ times and repeating the above
argument with $(X,Y)=(p_i, p_{i+1})$ completes the proof of the upper bound in Theorem \ref{T:main} when ${d_{in}}=2.$

The argument when ${d_{in}}\geq 3$ is essentially the same. The idea is to
take the diagrams depicted and rotate them around the axis $PZ.$
Lemma~\ref{L:extend} extends to higher dimensions with the triangle
$\triangle ABC$ replaced by the tip of a cone with the same diameter
requirement. Such a cone is obtained by rotating $X'ZY'$ in Figures 1
and 2. The rest of the argument then carries over verbatim. 

Now we analyze the efficiency of this procedure. First, to complete a
single radius increment requires covering the boundary of $B_r$ with
balls of radius $O\lr{\w^{-1}_f(\ep)}$. It is standard that in $\mathbb R^{d_{in}}$, this requires
 \[\left(\frac{O(r)}{\w^{-1}_f(\ep)}\right)^{{d_{in}}-1}\] 
balls. We get one extra max and min
in the max-min string we build to approximate $f$ for each such
ball. Thus, at a cost of $\left(\frac{O(r)}{\w^{-1}_f(\ep)}\right)^{{d_{in}}-1}$
many maxes and mins, the radius on which we approximate $f$ increases
\[r\mapsto  R_{r,\ep}=r+\frac{\w^{-1}_f(\ep)^2}{10r}.\] 
Hence, if we fix $R>\w^{-1}_f(\ep),$ then for every $\w^{-1}_f(\ep)\leq r \leq
R,$ we have
\[R_{r,\ep}-r\geq\frac{\w^{-1}_f(\ep)^2}{10R}\] 
and to obtain an approximation of $f$ on $B_R,$ by a max-min we need
to extend the approximation of $f$ from 
a small ball to a larger ball at most $10R^2/\w^{-1}_f(\ep)^2$
times. The number of maxes and mins required for each extension is
$(O(R)/\w^{-1}_f(\ep))^{{d_{in}}-1}$. Hence, the length of the max-min string we
construct to approximate $f$ on $B_R$ is 
\[\left(\frac{O(R)}{\w^{-1}_f(\ep)}\right)^{d_{in}+1},\]
as claimed.

\qed

\begin{remark}

We have tacitly neglected the case $d_{in}=1$. This case is the same as $d_{in}=2$ but easier. In fact here we require only \[\left(\frac{O(R)}{\w^{-1}_f(\ep)}\right)\]
layers, which would naively correspond to $d_{in}=0$. The reason is that a $1$-dimensional ball can be increased in radius by $\omega_f^{-1}(\varepsilon)$ by adding only a single external line segment of length $\omega_f^{-1}(\varepsilon)$. In the higher dimensional cases, we need to add the external pieces mostly tangentially which requires more layers. 

\end{remark}

\section{Proof of the Lower Bound in Theorem \ref{T:main}}\label{S:LB} The purpose of this section is to
prove that for every ${d_{in}}\geq 1,$ there exists $f\in C([0,1]^{d_{in}},\mathbb R)$ and $\eta=\eta({d_{in}},f)>0$ so $f$ satisfies the following property. For any $\Relu$ net $\mathcal N$
with input dimension ${d_{in}}$, hidden layer width ${d_{in}}$, and output dimension
$1$, we have
\[\norm{f-f_{\mathcal N}}_{C^0}\geq \eta.\]

In fact, we will show that if there is $a$ such that a compact connected component of the pre-image $f^{-1}(a)$ disconnects a bounded region from the infinite component of $\mathbb R^{d_{in}}$, then $f$ is not approximable by depth-${d_{in}}$ $\Relu$ nets.

Fix ${d_{in}}\geq 1,$ and consider a width ${d_{in}}$ $\Relu$ net
\[f_{\mathcal N} := A_n\circ \Relu \circ A_{n-1}\cdots \circ\Relu \circ A_1,\]
where the $A_i$'s are affine and $A_n$ maps $\R^{d_{in}}$ to $\R,$ while for
$1\leq i \leq n-1,$ the transformations $A_i$ map $\R^{d_{in}}$ to $\R^{d_{in}}.$ We may assume without loss of generality that $A_i$ have full rank for all $i$ since $f_{\mathcal N}$ is continuous with respect to the $A_i$'s and affine maps with full rank are dense among all affine maps. Define a \emph{level set} of a function $f$ to be a connected component of a pre-image $f^{-1}(a)$ for some $a$. The following Lemma shows the level sets of any function
\[f_j(x)=\Relu \circ  A_j\circ \cdots \Relu \circ A_1(x)\]
computed by the first $j$ hidden layers of $\mathcal N$ are of a
rather special form. 

\begin{lemma}
\label{L:LB-key}

For each $j\geq 1$, set $S_j$ to be the set of points on which all ReLU evaluations throughout the evaluation of $f_j$ are (strictly) positive. Then $S_j$ is open and convex, $f_j$ is affine on $S_j$, and every level set of $f_j$ that is bounded is contained in $S_j.$
\end{lemma}

\begin{proof}

For every $j,$ $S_j$ is open and convex since it is cut out by a collection of inequalities of the form $\set{\ell_k(x)>0},$ where $\ell_k:\R^{d_{in}}\gives \R$ are affine. Note that the level sets of $f_n:\R^{d_{in}}\gives \R$ are the union of level sets of $f_{n-1}:\R^{d_{in}}\gives \R^{d_{in}}.$ Hence, it is enough to show that for every $j\leq n-1$ if a level set of $f_j$ intersects the complement of $S_j,$ then it is unbounded. We prove this by induction on $j\geq 1.$ The base case $j=1$ is immediate if $S_1=\R^{d_{in}}.$ Otherwise, consider $y\not \in S_1.$ Then, at least one, say the $k^{th}$, component of $f_1(y)$ is zero. The inverse image under $\Relu$ of the point $f_1(y)$ therefore contains a ray (e.g. the ray $r$ starting at $f_1(y)$ and going to $-\infty$ parallel to the $k^{th}$ coordinate axis). The inverse image of $r$ under the affine map $A_1$ also contains a ray and hence is unbounded, proving the base case. 

For the inductive step, fix some $j\geq 2$. There is nothing to prove if $S_j=\R^{d_{in}}$.  Otherwise, consider $y \not \in S_j.$ If $y\not \in S_{j-1}$, then we are done by induction since $f_{j-1}^{-1}(f_{j-1}(y))\subseteq f_{j}^{-1}(f_{j}(y)).$ If $y\in S_{j-1},$ then we argue as before. Namely, the $k^{th}$ component of $f_{j}(y)$ vanishes for some $k$ and the inverse image under $\Relu \circ A_{j}$ of the point $f_{j}(y)$ therefore contains a ray. If this ray is contained in $f_{j-1}(S_{j-1})$, then its pre-image under $f_{j-1}$ also contains a ray since $f_{j-1}$ is affine when restricted to $S_{n-1}.$ Otherwise, this ray intersects the boundary of $f_{j-1}(S_{j-1})$ at some point $p.$ By induction, the pre-image $f_{j-1}^{-1}(p)$ is unbounded and hence so is $f_{j}^{-1}(f_{j}(y))$ since it contains $f_{j-1}^{-1}(p).$ This completes the proof.

\end{proof}

\noindent We now complete the proof of the lower bound in Theorem
\ref{T:main}. 
Suppose that $f:\R^{d_{in}}\gives \R$ is a continuous function such that for some $a$, the pre-image $f^{-1}(a)$ contains a compact connected component $A$ and that $\mathbb R^{d_{in}}\backslash A$ contains a bounded connected component $B.$ For example, we could take
\[f(x_1,\ldots, x_{d_{in}}):=\sum_{j=1}^{d_{in}} \left(x_j-\frac{1}{2}\right)^2,\qquad a=\frac{1}{4}.\]
In this case $A$ is a sphere and $B$ is a ball. Suppose $y\in B$ and $f(y)=b\neq a$. Then for $\eta<\frac{|a-b|}{4}$, we show that $f$ is not $\eta$-approximable on $A\cup B$. 

Set $c=\frac{a+b}{2}$, and let $C$ be the intersection of $f^{-1}(c)$ with $B$. Since $f$ is continuous, the intermediate value theorem implies that $C$ separates $y$ and $A$. Let $C'\subseteq C$ be the boundary of any connected component of $C$ that contains $y$. Informally, $A$ surrounds $C'$ which surrounds $y$. Now, suppose some $f_{\mathcal N}$ computed by a neural net satisfies 
\[\norm{f-f_{\mathcal N}}_{C^0(A\cup B)}\leq \eta.\]
Denote by $S_{\mathcal N}$ the set of points in $\R^{d_{in}}$ where all the $\Relu$s in $\mathcal N$ are positive. Suppose first that $S_{\mathcal N}$ contains $C'.$ By Lemma \ref{L:LB-key}, $S_{\mathcal N}$ is convex. Since the intermediate value theorem implies any path from $y$ to $\infty$ intersects $C'$, we know that $y$ is in the convex hull of $C'$. Hence we have $y\in S_{\mathcal N}$ as well. Since $f_{\mathcal N}$ is affine on $S_{\mathcal N},$ this means that $f_{\mathcal N}(y)$ is between the minimum and maximum values of $f_{\mathcal N}$ on $C'$. As $f(y)=b$ and $f(C')=c$ we get a contradiction since $\eta<\frac{|b-c|}{2}=\frac{|a-b|}{4}.$

Suppose in the second case that $S_{\mathcal N}$ does not contain $C'$, so there is $x\in C'\backslash S_{\mathcal N}$. Then, by Lemma \ref{L:LB-key}, the level set of $f_{\mathcal N}$ containing $x$ must be unbounded, and hence must intersect $A$ (as $A$ separates $y$ from $\infty$ and $x\in C'$ is reachable from $y$ without intersecting A). This is again a contradiction for $\eta<\frac{|a-c|}{2}=\frac{|a-b|}{4}.$ 

In both cases, we showed that $f$ and $f_{\mathcal N}$ differed significantly; the first case used the affineness of $f_{\mathcal N}$ on $S_{\mathcal N}$ while the second used the unboundedness of level sets away from $S_{\mathcal N}$. We conclude that a width-$d$ net cannot uniformly approximate $f$.
$\square$

\bibliographystyle{alpha}
  \bibliography{bibliography}

\end{document}